%% file: main_arxiv.tex
\documentclass{article}
\usepackage{arxiv}
\usepackage[numbers, compress]{natbib}

\usepackage[utf8]{inputenc} 
\usepackage[T1]{fontenc}    
\usepackage{hyperref}       
\usepackage{url}            
\usepackage{booktabs}       
\usepackage{amsfonts}       
\usepackage{nicefrac}       
\usepackage{microtype}      
\usepackage{amsthm}
\usepackage{mathtools}
\usepackage{algorithm}
\usepackage{algorithmicx}
\usepackage{algpseudocode}
\usepackage{multirow}
\usepackage{xspace}
\usepackage{subcaption}

\usepackage{amssymb}
\usepackage{MnSymbol}
\usepackage{amsfonts}
\usepackage{amsmath}
\usepackage{amsthm}
\usepackage{bm}
\usepackage[dvipsnames]{xcolor}
\usepackage{hyperref}
\hypersetup{
colorlinks = true,
linkcolor = Brown,
anchorcolor = Brown,
citecolor = Mahogany,
filecolor = RoyalBlue,
menucolor = Green,
runcolor = cyan,
urlcolor = Green}
\newcommand{\sn}[1]{}

\newcommand{\hima}[1]{}
\newcommand{\hideh}[1]{}

\newcommand{\actionable}{\mathcal{A}^p\xspace}
\newcommand{\bx}[0]{\mathbf{x}}
\newcommand{\bz}[0]{\mathbf{z}}
\newcommand{\xc}[0]{\mathbf{x}'}

\newcommand{\A}{\mathcal{A}}
\newcommand{\R}{\mathcal{R}}
\newcommand{\T}{\mathcal{T}}
\newcommand{\D}{\mathcal{D}}
\newcommand{\Own}{\mathcal{O}}
\newcommand{\ftheta}{f_{\btheta}}

\newcommand{\bdelta}[0]{\bm{\delta}}

\newcommand{\btheta}[0]{\bm{\theta}}
\newcommand{\bo}{\mathbf{o}}

\newtheorem{theorem}{Theorem}

\newtheorem{lemma}{Lemma}

\newtheorem{definition}{Definition}

\DeclareMathOperator*{\argmin}{arg\,min}
\title{On the Privacy Risks of Algorithmic Recourse}

\author{%
Martin Pawelczyk \\
University of Tübingen \\
\And
Himabindu Lakkaraju\thanks{Equal contribution} \\
Harvard University \\
\And 
Seth Neel\footnotemark[1] \\
Harvard University\\
\\
}

\begin{document}
\maketitle

\begin{abstract}
\input{abstract.tex}
\end{abstract}

\section{Introduction}
\input{intro.tex}
\section{Related Work} \label{section:related_work}
\input{related_work.tex}
\section{Preliminaries} \label{section:preliminaries}
\input{prelims.tex}
\section{Our Framework}
\label{section:theory}
\input{method}

\input{diff_privacy}

\section{Experimental Evaluation}  \label{sec:experiments}

\input{experiments.tex}

\section{Conclusion}  \label{section:conlusion}
\input{conclusion.tex}
\bibliography{bibliography}
\bibliographystyle{plainnat}

\newpage
\appendix
\input{supplement.tex}

\end{document}

%% file: abstract.tex
As predictive models are increasingly being employed to make consequential decisions, there is a growing emphasis on developing techniques that can provide algorithmic recourse to affected individuals.
While such recourses can be immensely beneficial to affected individuals, potential adversaries could also exploit these recourses to compromise privacy. In this work, we make the first attempt at investigating if and how an adversary can leverage recourses to infer private information about the underlying model's training data.
To this end, we propose a series of novel membership inference attacks which leverage algorithmic recourse. 
More specifically, we extend the prior literature on
membership inference attacks
to the recourse setting by leveraging the distances between data instances and their corresponding counterfactuals output by state-of-the-art recourse methods. 
Extensive experimentation with real world and synthetic datasets demonstrates significant privacy leakage through recourses.
Our work establishes unintended privacy leakage as an important risk in the widespread adoption of recourse methods. 


%% file: intro.tex
Machine learning (ML) models are increasingly being deployed in domains such as finance, healthcare, and policy to make a variety of consequential decisions. As a result, there is a growing emphasis on providing \emph{recourse} to individuals who have been adversely impacted by the predictions of these models~\cite{voigt2017eu}. For example, an individual who was denied a loan by a predictive model employed by a bank should be informed about what can be done to reverse this decision. 
Several approaches in the recent literature tackled the problem of providing recourse by generating counterfactual explanations ~\cite{wachter2017counterfactual,Ustun2019ActionableRI,karimi2019model,FACE,van2019interpretable} which highlight what features need to be changed and by how much to flip a model's prediction. 
For  instance,~\citet{wachter2017counterfactual} proposed a gradient based approach to find the nearest counterfactual resulting in the desired prediction. 
More recently,~\citet{karimi2021algorithmic,karimi2020probabilistic} 
advocated for leveraging the causal structure of the underlying data when generating recourses~\citep{barocas2020,mahajan2019preserving,pawelczyk2019}. 

As algorithmic recourses seep into real-world applications, adversaries could potentially exploit these recourses to extract information about the underlying models and their training data, thereby leaking sensitive information (e.g., a bank's customer data) and enabling fraudulent activities. 
Therefore, there is a clear and urgent need to investigate the privacy risks associated with algorithmic recourse. 
While there is extensive literature on privacy attacks and defenses for machine learning models~\cite{rigaki2020survey,abadi2016deep,papernot2018scalable,shokri2017membership}, there is very little research~\cite{shorki2021explanation,aivodji2020model} that focuses on the privacy risks that arise when adversaries have access to \emph{explanations} which highlight the rationale behind one or more model predictions. 
Recently, \citet{shorki2021explanation} studied if and how feature attribution based explanations (which capture the feature importances associated with individual model predictions) leak sensitive information, and ~\citet{aivodji2020model} developed model extraction (i.e., reconstructing the underlying model) attacks against counterfactual explanations. However, neither of these works explore if and how adversaries may leverage recourses to infer sensitive information about the underlying model's training data.
In this work, 
we address the aforementioned gaps by initiating a study of \emph{if and how an adversary can leverage algorithmic recourses to leak sensitive information about the training data of the underlying model}. 
To this end, we introduce a general class of membership inference attacks called \emph{counterfactual distance-based attacks} which leverage algorithmic recourse to determine if an instance belongs to the training data of the underlying model or not. In formulating this new class of attacks, we exploit the intuition that the distance between an instance and its corresponding recourse may capture information about whether that instance was used to train the model. We instantiate the aforementioned class of attacks to propose two novel membership inference attacks. Our first attack infers membership by thresholding on the distance between a given instance and its corresponding algorithmic recourse. Our second attack draws inspiration from state-of-the-art loss-based membership inference attacks~\cite{carlini2021, enhanced_mi} and proposes a likelihood ratio test (LRT) that accounts for algorithmic recourse. 
Our attacks operate under the assumption that the adversary can only query the recourse algorithm once. This assumption is a lot more practical than those considered in related works~\cite{aivodji2020model}, and is inline with real-world settings where an end user would typically be provided with a single recourse and will not be able to query the underlying model or recourse algorithm multiple times~\cite{Ustun2019ActionableRI,wachter2017counterfactual}. \emph{To the best of our knowledge, our work is the first to introduce membership inference attacks which leverage algorithmic recourse}. 

We experiment with multiple real world datasets spanning diverse domains such as lending, healthcare, and law to evaluate the effectiveness of the proposed attacks. 
Our experimental results clearly demonstrate the efficacy of the proposed attacks, and highlight significant privacy leakage through recourses generated by a wide range of recourse algorithms. In addition, the proposed attacks also outperform the state-of-the-art loss-based membership inference attacks (which do not leverage recourses) on data with sufficiently high dimensionality, thus highlighting the promise of our recourse-based attacks as generic membership inference attacks. We also empirically analyze the factors contributing to the success of our attacks, and find that our attacks are highly successful when the underlying model overfits to the training data~\cite{carlini2021} and the dimensionality of the data is high. Overall, our results establish unintended privacy leakage as an important risk in the widespread adoption of recourse algorithms. 

%% file: related_work.tex
\textbf{Algorithmic Recourse.}
Several approaches have been proposed in literature to provide recourse to individuals who have been negatively impacted by model predictions \citep{tolomei2017interpretable,laugel2017inverse,Dhurandhar2018,wachter2017counterfactual,Ustun2019ActionableRI,joshi2019towards,van2019interpretable,pawelczyk2019,mahajan2019preserving,mothilal2020fat,karimi2019model,rawal2020interpretable,karimi2020probabilistic,dandl2020multi,antoran2020getting,spooner2021counterfactual}. 
These approaches can be broadly categorized based on \citep{verma2020counterfactual}: 
\emph{type of the underlying predictive model} (e.g., tree vs.\ differentiable classifier), \emph{type of access} they require to the underlying predictive model (e.g., black box vs.\
gradient access), whether they encourage \emph{sparsity} in counterfactuals (i.e., only a small number of features should be changed), whether counterfactuals should lie on the \emph{data manifold}, whether the underlying \emph{causal relationships} should be accounted for when generating counterfactuals, and whether the output produced by the method should be \emph{multiple diverse counterfactuals} or a single counterfactual. 
Recent research also highlighted and addressed various challenges pertaining to the robustness~\cite{dominguezolmedo2021adversarial,slack2021counterfactual,pawelczyk2022algorithmic,rawal2021modelshifts} 
and fairness~\cite{gupta2019equalizing,von2020fairness} 
of algorithmic recourse. 
However, none of the aforementioned works explore the privacy risks associated with algorithmic recourse which is the focus of our work. 

\textbf{Privacy Attacks for ML Models.}
There is a long line of prior work developing privacy attacks on machine learning models~\cite{rigaki2020survey,shokri2017membership,sablayrolles2019whitebox, carlini2021, enhanced_mi}. 
One class of attacks called \emph{membership inference attacks} focus on determining if a given instance is present in the training data of a particular model~\cite{shokri2017membership,sablayrolles2019whitebox, carlini2021, enhanced_mi}.
These attacks typically exploit the differences in the distribution of model confidence on the true label (or the loss) between the instances that are in the training set and those that are not. For example,~\citet{shokri2017membership} proposed a loss-based membership inference attack which determines if an instance is in the training set by testing if the loss of the model for that instance is less than a specific threshold. Other membership inference attacks are also predominantly loss-based attacks where the calibration of the threshold varies from one proposed attack to the other~\cite{sablayrolles2019whitebox, carlini2021, enhanced_mi}. 

\hima{Martin TO DO}
\hima{cut this down; just cite 6, say there is work that leverage additional info beyond loss functions to do membership inference attacks. for instance, 6 leverages adversarial examples to orchestrate these attacks. However, these pieces of information are less likely to be available/accessible to general public unlike recourse where the main goal is to hand them over to various kinds of customers/consumers but none of these works including 6 focus on recourse. don't say anything about technical similarities and differences here.}

Some works leverage additional information beyond loss functions to do membership inference attacks. For instance, \cite{label-only} leverages adversarial examples to orchestrate membership inference attacks.
While there exist similarities between adversarial examples and recourses output by \texttt{SCFE} \cite{wachter2017counterfactual}, the recourses output by other SOTA methods such as \texttt{CCHVAE} \cite{pawelczyk2019} and \texttt{GS} \cite{laugel2017inverse} are quite different from adversarial examples and their framework does not apply to algorithmic recourse broadly.

\textbf{Intersections between Privacy and Explainability.} The intersection between privacy and explainability, which is the focus of our work, is relatively under explored. Recently,~\citet{shorki2021explanation} 
developed a membership inference attack against a variety of feature attribution based explanation methods (e.g., LIME~\cite{ribeiro2016should}, SHAP~\cite{lundberg2017unified}, Gradient-based methods~\cite{sundararajan2017axiomatic,smilkov2017smoothgrad}). 
Their attack exploits the intuition that the higher the variance of the feature attribution corresponding to an instance, the more uncertain the corresponding prediction (i.e., higher model loss), and therefore the less likely it is that the instance belongs to the training set. Furthermore,~\citet{aivodji2020model} developed a model extraction attack against counterfactual explanation methods. Their attack leverages model predictions and counterfactual explanations (output in case of unfavorable predictions) corresponding to a set of instances, and learns a proxy model that mimics the behavior of the model under attack as closely as possible. None of these works develop membership inference attacks against counterfactual explanation (algorithmic recourse) methods which is the main theme of our work. 

%% file: prelims.tex
Let us consider a predictive model $\ftheta:\mathcal{X}\to\mathcal{Y}$ where $\mathcal{X}\subseteq\mathbb{R}^d$ is the feature space, $\mathcal{Y}$ is the space of outcomes, and $\theta \in \Theta$ denotes the parameters of the model $\ftheta$. 
Let $\mathcal{Y} = \{0,1\}$ where $0$ and $1$ denote an unfavorable outcome (e.g., loan denied) and a favorable outcome (e.g., loan approved) respectively. In practice $\ftheta(x)$ will output a probability in $[0,1]$ of a positive classification, which is then thresholded to obtain a binary classification, and we will denote by $(\ftheta(x))_{y}$ the probability $\ftheta$ assigns to a binary label $y$. Let us assume that the model $\ftheta$ was trained using some data set $D_{t} = (X_{t},Y_{t})$ where each $(x,y) \in D_t$ is sampled from an underlying data distribution $\mathcal{D}$. A training algorithm $\T: (\mathcal{X} \times \mathcal{Y})^{n} \to \Theta$ is a potentially randomized algorithm that takes in a dataset $D_t$ and outputs a model $\ftheta$. With this notation in place, we provide an overview of the standard formulations for algorithmic recourse as well as membership inference attacks.

\paragraph{Algorithmic Recourse.}
Let $x \in \mathcal{X}$ be an instance which received a negative outcome i.e., $\ftheta(x) = 0$. The goal here is to find a recourse for this instance $x$ i.e., to determine a set of changes $\delta$ that can be made to $x$ in order to reverse the negative outcome. The problem of finding a recourse for $x$ involves finding a counterfactual $x' = x + \delta$ for which the predictive model outputs a positive outcome i.e., $\ftheta(x') = \ftheta(x+\delta) = 1 $.
Note that it is desirable to minimize the cost $c(x, x')$ required to change $x$ to $x'$ so that the recourse is easily implementable. 
In practice, $\ell_1$ or $\ell_2$ distance are commonly used as cost functions \cite{wachter2017counterfactual}. 
Furthermore, since recommendations to change features such as gender or race would be unactionable, it is important to restrict the search for counterfactuals so that only actionable changes are allowed. Let $\actionable$ denote the set of plausible or actionable counterfactuals. 

Putting it all together, the problem of finding a recourse for instance $x$ for which $\ftheta(x) = 0$ can be formalized as~\cite{wachter2017counterfactual}:
\begin{align}
x' &= \argmin_{\substack{x' \in \actionable}} \ell(\ftheta(x'),1) + \lambda \cdot c(x, x'),
\label{eqn:generalrecourse_un}
\end{align}
where $\ell:\mathcal{X} \times \mathcal{Y} \to \mathbb{R}_{+}$ denotes a differentiable loss function (e.g., binary cross entropy loss) which ensures that gap between $\ftheta(x')$ and the favorable outcome $1$ is minimized, and $\lambda>0$ is a trade-off parameter. 
Eqn.\ \eqref{eqn:generalrecourse_un} captures the generic formulation leveraged by several of the state-of-the-art recourse finding algorithms~\cite{wachter2017counterfactual}. In general, we denote (potentially randomized) recourse algorithms as $\R: (\Theta, \mathcal{X}) \to \mathcal{S}$.
In the standard recourse setting $\R$ returns a recourse $x'$ and so $\mathcal{S} = \mathcal{X}$.

\input{tables/attack_assumptions_summary}
\paragraph{Membership Inference Attacks for ML Models}\label{subsec:mia}
The goal of a membership attack is to create a function that accurately determines if an instance $z = (x,y)$ belongs to the training set of the model $\ftheta$. Several membership inference attacks proposed in literature exploit the intuition that models have lower loss on instances that were observed during their training. Such approaches are commonly referred to as \emph{loss-based attacks}. Below, we discuss two of the most popular loss-based attacks developed by~\citet{yeom2017privacy} and~\citet{carlini2021}. 

\emph{Thresholding on Model Loss~\cite{yeom2017privacy}.} This attack takes an instance $x$ and determines whether it is a member of the training set (\texttt{MEMBER}) by checking if the \texttt{Loss} of the model $\ftheta$ 
is lower than or equal to a threshold $\tau_L$:
\begin{align}
M_{\text{Loss}}(x) = 
\begin{cases}
\emph{\texttt{MEMBER}} & \text{ if } \ell(\btheta,z) \leq \tau_L \\
\emph{\texttt{NON-MEMBER}} & \text{ if } \ell(\btheta,z) > \tau_L.
\end{cases}\label{eqn:threshold-loss}
\end{align}
\citet{sablayrolles2019whitebox} demonstrated that this attack is nearly optimal in the sense that it is approximately equivalent to the likelihood ratio test under certain conditions, which is the uniformly most powerful test for a given significance level (by the Neyman-Pearson Lemma). Note that this attack is only feasible when the adversary has access to the true label $y$ of $x$, and the model's loss function $\ell$ and its parameters $\theta$.

\emph{Likelihood Ratio Attack~\cite{carlini2021}.} Recent work has attempted to further approximate a test based on the full likelihood ratio of the model $\theta$ by computing the likelihood ratio of the model loss, or equivalently confidence, $\text{conf}(f_{\theta}, z)$. Given sample access to $\mathcal{D}$ the adversary trains shadow models, and computes $\text{conf}(f_{\theta}, z)$ in the case when $z$ is included in the training set and when it is a test point. Under the assumption that the logit scaled confidence is normally distributed, normal distributions are fit to the ``in'' and ``out'' scaled confidences, and an approximate likelihood ratio $\Lambda =  \frac{ \Pr[\text{conf}(\ftheta, z) \text{ } | \text{ }  \mathcal{N}(\mu_{in}, \sigma^2_{in})] } { \Pr[\text{conf}(\ftheta, z) \text{ } | \text{ } \mathcal{N}(\mu_{out}, \sigma^2_{out})] }$ is computed. Finally, the adversary predicts \texttt{MEMBER} when $\Lambda > \tau$.
We call this attack \texttt{Loss LRT}.
In Subsection~\ref{subsec:attack_thresh} we develop an LRT attack based not on $\text{conf}(f_{\theta}, z)$, but on the \emph{counterfactual distance}, which does not require direct access to $f_\theta$ or $y$ and can be implemented with algorithmic recourses.

%% file: tables/attack_assumptions_summary.tex
\begin{table}[tb]
\centering
\begin{tabular}{ccccc}
\toprule
 Info & \texttt{Loss} & \texttt{CFD} & \texttt{Loss LRT} & \texttt{CFD LRT}  \\
 \cmidrule(lr){1-1}  \cmidrule(lr){2-5}
Query access to $f_{\btheta}$ & $\checkmark$ & $\times$ & $\checkmark$ & $\times$  \\
Query access to $\R$ & $\times$ & $\checkmark$ & $\times$ & $\checkmark$ \\
Known loss function & $\checkmark$ & $\times$ & $\checkmark$ & $\times$ \\
Access to $\D^N$ & $\times$ & $\times$ & $\checkmark$ & $\checkmark$  \\
Access to true labels & $\checkmark$ & $\times$ & $\checkmark$ & $\times$ \\
\bottomrule
\end{tabular}
\caption{Summarizing the assumptions underlying the different MI attacks. The recourse based attacks do not require access to the true labels nor do they need to know the correct loss functions.}
\vspace{-0.2in}
\label{tab:summary}
\end{table}

%% file: method.tex

In this section, we introduce a general class of novel membership inference attacks that leverage algorithmic recourse.
First, we introduce the recourse-based membership inference game which generalizes the previously proposed membership inference 
attacks~\cite{yeom2017privacy} on ML models to account for information captured in algorithmic recourse (see Table \ref{tab:summary} for an overview).
Then we introduce an attack that uses the distance between the recourse and the input (Subsection~\ref{subsec:distance}), and show how the attack can be improved with a likelihood-ratio test (LRT) based approach in the style of \cite{carlini2021} which we present in Subsection~\ref{subsec:attack_thresh}. 

\begin{definition}[Recourse-based MI Game]
The game has two players: a model owner ($\Own$) and an adversary ($\A$). The players take the following actions:
\vspace{-0.1in}
\begin{itemize}
\item[(1)] $\Own$ draws a training set from the population $D_t \sim \D^{N}$, and using training algorithm $\T$, equipped with loss function $\ell$ trains a model $f_{\btheta} \sim \T(D_t)$. $\Own$ then labels every point $z \in D_t$ with a binary label $f_{\btheta}(z)$. Let $D_t^0$ denote the subset of the training data with $f_{\btheta}(x) = 0$, and let $\D^{\theta, 0}$ denote the conditional distribution $p(z) \sim D | f_{\btheta}(z) = 0$.  $\Own$ flips a coin: if ``heads'', $\Own$ samples $x \sim \D^{0, \theta}$, else $x \sim D_t^{0}$. Using recourse generation algorithm $\R$, $\Own$ generates a recourse $x' \sim \R(f_{\btheta}, x, D_t)$ for $x$. Then $\Own$ sends $s = (x', x)$ to $\A$.
\item[(2)] In addition to $s$, $\A$ obtains query access to $\D$. We assume that $\A$ has full knowledge of all the implementation details of $\Own$, including the specifics of $\T$ and $\R$. Finally $\A$ produces a binary guess $G$ indicating whether $x \in D_t$ (\texttt{MEMBER}) or $x \not \in D_t$ (\texttt{NON-MEMBER}). 
\end{itemize}
\label{defintion:mi_game_recourse}
\end{definition}
\sn{I put in a few sentences about evaluation next; } \hima{TODO Hima to take a on this paragraph as well.}
\vspace{-0.1in}
 We now present two attacks based on a statistic we call the \emph{counterfactual distance}.

\subsection{Thresholding on Counterfactual Distance}\label{subsec:distance} \hima{TODO: Hima to take a pass on 4.1 to clean up.} In the recourse-based MI Game, all the adversary has access to is the original instance $x$ and its counterfactual $x'$ generated by recourse algorithm. Loss-based attacks perform well at determining whether a point is a \texttt{MEMBER} of the training set or not, because the model typically over-fits to the training points, leading to lower losses on these points than on the test set. One explanation for this loss-disparity that has been given in prior work \cite{shorki2021explanation}, is that during the training process, the decision boundary is forced away from training points. This suggests that points in the training set should be further from the boundary than points in the test set, motivating a distance-based attack that predicts a point is a \texttt{MEMBER} of the training set if its loss is below some threshold $\tau$. The distance of a point $x$ to the boundary can be computed as $c(x, x')$ where $x'$ is the solution to Equation~\ref{eqn:generalrecourse_un} with $\mathcal{A}^{p} = \mathbb{R}^d$.
So if we exactly optimize the objective function in Equation~\ref{eqn:generalrecourse_un} that underpins our recourse algorithms (with $\actionable = \mathbb{R}^d$), the counterfactual distance $c(x, x')$ is exactly the distance to the model boundary. While algorithms that focus on generating \emph{realistic} recourses~\cite{pawelczyk2019,karimi2021algorithmic} do not exactly optimize this objective in general, we can still view the distance to the recourse as a proxy for the distance of $x$ to the model boundary. Based on this intuition we have the following counterfactual distance (\texttt{CFD}) based attack:
\begin{align}
M_{\text{Distance}}(x) = 
\begin{cases}
\emph{\texttt{MEMBER}} & \text{ if } c(x,x') \geq \tau_D \\
\emph{\texttt{NON-MEMBER}} & \text{ if } c(x,x') < \tau_D.
\end{cases}\label{eqn:threshold-dist}
\end{align}
Following \cite{carlini2021}, we assume for the first two attacks below that $\A$ knows apriori an optimal threshold $\tau_{\alpha}$ that maximizes a given TPR subject to a fixed FPR $\alpha$, as the purpose of these simple attacks is to illustrate the potential privacy leakage through the recourse output. In practice, we will be plotting the TPR vs. FPR curves over all values of the threshold $\tau_{\alpha}$ and so we will not need to pick a specific one.

\subsection{Likelihood Ratio Test using Counterfactual Distance}\label{subsec:attack_thresh}
\begin{algorithm}[!ht]
\begin{algorithmic}[1]
\caption{One-sided Distance-based Likelihood Ratio Test (\texttt{CFD LRT})}
\label{alg:thresh_lrt}
\State \textbf{Inputs:} point $(x, y)$, recourse output $s = \texttt{GetRecourse}(x, f_{\btheta}), \D$; FP-Rate: $\alpha$, \# Shadow Models: $N$, $\T = \texttt{TrainClassifier}(\cdot)$
\State $\text{teststats} = []$
\State \textbf{Compute}: $t_0 = T(s) = c(x,x')$ 
\For{$i=1:N$} 
\State Sample $\mathcal{D}_t^{(i)} \sim \D$
\State $f_{\btheta^{(i)}} = \texttt{TrainClassifier}(\mathcal{D}^{(i)})$
\State $s^{(i)} = \texttt{GetRecourse}(x, f_{\btheta^{(i)}})$
\State teststats $\xleftarrow{} T(s^{i}) = c(x, x'^{(i)})$
\EndFor
\State $\hat{\mu}_{\text{MLE}} = \frac{1}{N} \sum_{i=1}^N \big(\log c(x, \xc^{(i)}) \big)  \big)$
\State $\hat{\sigma}^2_{\text{MLE}}  = \frac{1}{N} \sum_{i=1}^N \big(\hat{\mu}_{\text{MLE}} - \log\big(c(x, \xc^{(i)})\big) \big)^2$
\If{$t_0$ > $z_{1-\alpha}$} \Comment{ $z_{1-\alpha}$ is the $1{-}\alpha$-quantile of $Z \sim \mathcal{LN}(\hat{\mu}_{\text{MLE}}, \hat{\sigma}^2_{\text{MLE}}$})
    \State \textbf{Output:} $G$ = \texttt{\texttt{NON-MEMBER}} 
\Else
    \State \textbf{Output:} $G$ = \texttt{MEMBER} 
\EndIf
\end{algorithmic}
\end{algorithm}
\citet{carlini2021} showed that \texttt{Loss LRT} attacks perform better than simple \texttt{Loss} thresholding. In Algorithm~\ref{alg:thresh_lrt} we present an LRT version of our counterfactual distance attack (\texttt{CFD LRT}).
As in prior work \cite{carlini2021, sablayrolles2019whitebox} since computing the LRT (Equation~\ref{eq:cfd_lrt}) exactly is intractable, we  make several modifications to the attack that allow us to compute it efficiently.
The full likelihood ratio given $c(x,x')$ is:
 \begin{equation}
 \label{eq:cfd_lrt}
     \Lambda =  \frac{ \Pr[c(x, x')|x \in D_t]} {\Pr[c(x, x')|x \not \in D_t]}.
 \end{equation}
 We model the distributions of our statistic $c(x, x')$ as log-Normal, and so in order to compute eqn.\ \eqref{eq:cfd_lrt}, we need to estimate $(\mu_{in}, \sigma_{in}), (\mu_{out}, \sigma_{out})$ where we assume that if $D_t\backslash \{x\} \sim \mathcal{D}, \theta \sim \mathcal{T}(\{x\} \cup D_t), x' \sim \mathcal{R}(x, \theta, D_t)$, then $\log c(x,x') \sim \mathcal{N}(\mu_{in}, \sigma_{in})$, and similarly when $x' \not \in D_t$, $\log c(x,x') \sim \mathcal{N}(\mu_{out}, \sigma_{out})$. Given access to $\mathcal{D}$ we can estimate the parameters $\mu, \sigma$ by drawing fresh datasets, training models $\theta$ -- with or without a given point $x$ -- and then computing the resulting counterfactual distances (Lines $5-8$). However, as in \cite{carlini2021} we note that to approximate the numerator, we have to perform this sampling and model training separately for each $x$ that we perform the attack on, which is computationally infeasible. Hence in Algorithm~\ref{alg:thresh_lrt} we present a one-sided version of the LRT, where in Lines $10-11$ we estimate $\mu_{out}, \sigma_{out}$, and our attack predicts $\texttt{MEMBER}$ if $c(x,x')$ has a sufficiently low likelihood under these parameters. Note that since $\mu_{out}, \sigma_{out}$ do not depend on $x$, we only need to perform the process of training Shadow models once, even if we are evaluating our attack on many different $x$'s.

%% file: diff_privacy.tex
\subsection{Is privacy leakage through recourses inevitable?}\label{subsec:priv}
The attacks developed above and empirical results in Section~\ref{sec:experiments} suggest that recourses can be exploited to infer private information about the underlying training set. This raises a natural question: Is privacy leakage through recourses inevitable? 

Over the last decade, differential privacy \cite{privacybook} has emerged as the canonical approach to provably preventing membership inference for a wide array of statistical tasks. Applying this to the recourse setting, results which have been folklore in the privacy community imply that if the recourse generation algorithm is DP in the training data, we can provably bound the success of \emph{any} adversary $\mathcal{A}$ in the Recourse-based MI Game. In Theorem~\ref{priv_thm} (proof deferred to Supplement) we state a variant of the folklore result tailored to our setting, showing that not only can we bound the excess accuracy of the adversary over random guessing,  we can also bound the balanced accuracy (BA). Since $\text{BA} = \frac{\text{TPR} + \text{TNR}}{2}$, this implies that for a small FPR $\alpha$, the TPR of $\mathcal{A}$ is also close to $\alpha$. Recent work advocates for evaluating the success of MI attacks at low FPR instead of just looking at the overall accuracy \cite{enhanced_mi, carlini2021}. 

\begin{theorem}\label{priv_thm}
Let $\T: (\mathcal{X} \times \mathcal{Y})^{n} \to \Theta$ denote the training algorithm, draw $D_t \sim \D^n$ and  and $\A$ be an arbitrary adversary that receives $z = (x, y), s \sim \R(f_{\theta},x, D_t)$  from the recourse inference game, and produces a guess $G \in \{\text{MEMBER}, \text{NON-MEMBER}\}.$ Then, if $\mathcal{R}$ is $(\epsilon,0)$-differentially private, we have for all $\A$:
\vspace{-0.1in}
$$\text{BA}_{\mathcal{A}} \leq \frac{1}{2} + \frac{1-e^{-\epsilon}}{2}.$$
\end{theorem}
\vspace{-0.1in}
While Theorem~\ref{priv_thm} provides strong privacy guarantees, for several reasons both generic and specific to the recourse setting, \emph{differential privacy is not a silver bullet to defend recourses against membership inference attacks}. It is known that training with DP causes a significant drop in accuracy on even relatively simple benchmarks \cite{temp_sig}, and so when accuracy is a concern this defense may not be feasible. Moreover, model accuracy aside, private training could alter the distance between training points and the model boundary, potentially leading to costlier and less actionable recourses for individuals.

%% file: experiments.tex

Here, we discuss the detailed experimental evaluation of our proposed attacks.
Relative to related work \citep{label-only}, we are focusing on (i) a variety of recourse algorithms and (ii) evaluate our suggested LRT based attacks.
We do so by first comparing the proposed recourse-based attacks to each other using log-scaled AUC curves, that emphasize the importance of the low false-positive rate regime.
Recently, the latter metric has been advocated for \cite{enhanced_mi, carlini2021}.
Second, we use average-case metrics such as AUC or balanced accuracy (BA) to understand which recourse algorithms are most vulnerable to membership inference attacks.
We show these results in Section \ref{subsec:eval_attack_efficacy}.
Finally, in Section \ref{subsec:understanding} we leverage synthetic data to analyze the determining factors of attack success. \sn{Rewrite this whole summary with a bit more specifics layout the different sections. We don't need to discuss evaluation metrics here as we cover it below}
Below, we describe our experimental setup in more detail.
\input{tables/ann_performance_realworld}
\input{figures/figures_ann_realworld} 
\subsection{Setup}
\paragraph{Datasets}
Our first data set is the \emph{Adult} (A) data set \cite{Dua:2019} that originates from the 1994 Census database, consisting of 14 attributes and 48,842 instances. The class label indicates whether an individual has an income greater than 50,000 USD/year. Second, we use the \emph{Home Equity Line of Credit} \emph{(Heloc)} (H) data set ($d=23$). Here, the target variable records a score indicating whether individuals will repay the Heloc account within a fixed time window. Across both tasks we consider individuals in need of recourse if their scores lies below the median score across the entire data set, thresholding the scores based on the median to obtain binary target labels. 
Third, we use the \emph{Diabetes} (D) data set which contains information on diabetic patients from 130 different US hospitals \citep{strack2014impact}. 
The patients are described using administrative (e.g., length of stay) and medical records (e.g., test results) ($d=42$), and the prediction task is concerned with identifying whether a patient will be readmitted within the next 30 days. 
In line with \citep{shorki2021explanation}, we sub sampled smaller data sets of 10000 points from each of these datasets. 5000 points are left to the model owner to train the their private model, while another 5000 data points are used by the adversary to train their shadow models when applicable. 

Additionally, we follow \citet{shorki2021explanation} and generate synthetic data sets from the sklearn library. 
For $d$ features the method randomly chooses a vertex from the $d$-dimensional hypercube as a center for each of the classes, and samples Gaussian distributed random variables centered at the vertex with unit variance. 
For the linear model experiments, we use $d\in \{100, 1000, 5000, 7000\}$ with $n=5000$ training and test samples. 
For the non-linear models we use $d\in \{50, 150\}$.

\textbf{Recourse Algorithms and Predictive Models.}
We apply our techniques to three different methods which aim to generate low-cost recourses using different principles: \texttt{SCFE} is the method suggested by \citet{wachter2017counterfactual} and uses a gradient-based objective to find recourses, \texttt{GS} conducts a random search for recourse in the input space \citep{laugel2017inverse}, and \texttt{CCHVAE} \citep{pawelczyk2019} searches for recourse in a lower dimensional latent space using a generative model to encourage recourses to lie on the data manifold.
All methods use a $\ell_1$-regularizer to encourage sparse recourses.
We use implementations from the \texttt{CARLA} library \citep{pawelczyk2021carla}.
For all data sets, we trained a fully connected neural network with one hidden layer and 1000 nodes. 
The network uses ReLU activation functions and are trained using the ADAM optimizer. 
The network is trained for 250 epochs.
Across all tasks, we use a learning rate of 0.0001.
All recourses were generated with respect to this classifier. 

\paragraph{Baseline Attacks.}
We implement several baseline attacks: 
the baseline of random guessing that for a target FPR $\alpha$ predicts \texttt{MEMBER} with probability $\alpha$, and baselines based on the loss $\ell$.
The loss-based baselines are simple thresholding on the loss \cite{yeom2017privacy} (i.e., \texttt{Loss}), and the offline loss-based LRT \cite{carlini2021} (i.e., \texttt{Loss LRT}).

\paragraph{Evaluation Measures.} We use several well established measures to validate the efficacy of our proposed membership inference attacks. Consistent with with previous works \citep{yeom2017privacy,shokri2017membership,shorki2021explanation} we report balanced accuracy (BA) and receiver operating characteristic (ROC) area under the curve (AUC) scores.  Additionally, we follow \citep{carlini2021,enhanced_mi} and also report log-scale ROC curves, and true positive rates of the attacks at low false positive rates. The authors argue that, for membership inference attacks, average case metrics such as BA and AUC are not well suited. 
The underlying idea is that if a membership inference attack can  identify even a very small subset of the training data with very high confidence, then the attack should be considered successful. 
We follow this intuition and primarily report our findings using this metric.




\subsection{Evaluating the Attack Efficacy} \label{subsec:eval_attack_efficacy}
\vspace{-0.1in}
Inspecting Figure~\ref{fig:real_world_all_methods}, we see that for nearly all datasets across all methods, at sufficiently low FPR the \texttt{CFD LRT} curve lies above the diagonal, outperforming the random baseline; the one exception being the Adult dataset with recourses generated from \texttt{GS}. \texttt{CFD} also often outperforms the random baseline in most cases, but not on the Diabetes dataset with recourse method \texttt{SCFE} or the Heloc dataset with method \texttt{CCHVAE}. These trends are reflected in the metrics captured in Table~\ref{table:comparing_efficacy}, where \texttt{CFD LRT} achieves TPR $> .01$ at FPR $= .01$, and AUC $> .5$ in $7$ of the $8$ dataset-recourse settings. \texttt{CFD} achieves TPR $> .01$ at FPR $= .01$ in $4$ of the $8$ settings. Focusing on the metric of TPR at FPR $= .01$, we see that in the $7$ of $8$ settings where either method outperforms the random baseline, \texttt{CFD LRT} achieves higher TPR than \texttt{CFD} $5$ times. This difference is particularly evident on the Heloc dataset, where \texttt{CFD LRT} obtains TPR $35\%, 270\%$, and $413\%$ above the random baseline for \texttt{SCFE}, \texttt{GS}, \texttt{CCHVAE} respectively. In summary, both methods often outperform the random baseline across all metrics, showing substantial privacy leakage from algorithmic recourses, with \texttt{CFD LRT} generally outperforming \texttt{CFD}. The results also show that the Heloc dataset is particularly vulnerable to distance-based attacks across all recourse algorithms, with the Diabetes dataset being less susceptible. On the Adult and Diabetes datasets the distance-based attacks usually outperform random guessing, however, the improvement over the random baseline is less pronounced. 

Perhaps most interesting, is that only in the case of \texttt{CCHVAE} do the \texttt{CFD LRT} attacks that outperform the random baseline and are plotted in Figure~\ref{fig:real_world_all_methods}(c), actually reverse the direction of the threshold. This means that for the \texttt{CFD LRT} curves in Figure~\ref{fig:real_world_all_methods}(c) and metrics in Table~\ref{table:comparing_efficacy} in column \texttt{CCHVAE}, rather than predicting \texttt{MEMBER} in Line $12$ of Algorithm~\ref{alg:thresh_lrt} if $c(x,x') > z_{1-\alpha}$, we predict \texttt{MEMBER} iff $c(x,x') < z_{\alpha}$. While these results stand in contrast to our findings for \texttt{SCFE} and \texttt{GS}, there is an intuitive explanation.   \texttt{CCHVAE} trains a VAE to model the data generating distribution, and performs sampling \emph{in the latent space} to find a point in latent space $z'$ that is close to the representation of $x$ in the latent space $\text{encode}(x) = z_{x}$. Then it outputs the recourse $x' = \text{decode}(z')$ in the input space. Regardless of the specific $z'$ found by the recourse algorithm, $\text{decode}(z')$ is still in the range of the generative model. It is a known property of generative models like VAEs and GANs that their generated samples tend to be closer to training points than to test points, a fact which has been exploited for MI \cite{gan_survey}. One explanation for the results in Figure~\ref{fig:real_world_all_methods}(c) is that this property of generative models is in some cases outweighing the effect of the optimization during model training.

\subsection{Towards Understanding Attack Success}
\label{subsec:understanding}
To better understand the factors underlying these impressive results, in this section we use experiments on synthetic data with $\mathcal{R} = \texttt{SCFE}$ to examine the role of model type, model size, and number of features in attack success. 

\textbf{Linear Models.}
One potential confounding factor in assessing the results for \texttt{SCFE}, is that when $f_\theta(x)$ is non-convex so is the objective in Equation~\ref{eqn:generalrecourse_un}, and so there is no guarantee that the recourse output is close to globally optimal. In the case when $f_{\theta}$ is linear the objective is convex and has a closed form minimizer \cite{pawelczyk2022algorithmic, label-only}. In Figure~\ref{fig:linear_synth_experiments} we train logistic regression models on the synthetic data, varying the number of features, and plotting the distribution of the counterfactual distances for train and test points, as well as the log-scaled ROC curves. We observe that the training distance distribution starts moving away from the test distance distribution as the number of feature dimensions increases (see Figure \ref{fig:linear_distances}), and as expected the distance-based attack starts performing better as the number of features increases (see Figure \ref{fig:linear_tpr_fpr}). 
Strikingly, the \texttt{CFD LRT} in particular not only outperforms the \texttt{CFD} and random baseline, but also thresholding based on the \texttt{Loss} and the \texttt{Loss LRT} --- attacks which have access to the full loss function and the label $y$. Figures~\ref{fig:linear_distances}, \ref{fig:linear_tpr_fpr} suggest two potential reasons why the results in Figure~\ref{fig:real_world_all_methods} for \texttt{SCFE} do not exhibit the same level of attack performance as on the synthetic data: 
(i) Since the model classes are non-convex it could be due to our optimization failing to find the recourse that minimizes Equation~\ref{eqn:generalrecourse_un}, and (ii) Distance-based membership inference attacks are more effective in high dimensions, and the tabular datasets we experiment on are of relatively small dimension $d < 50$.

\input{figures/figures_linear_synth}
\textbf{Nonlinear models.}
\input{figures/figures_ann_synth}
So far we have studied the effect of the dimension on the attack success, which in the case of linear models is equal to the number of parameters. 
To disentangle the effect that the number of features has from the effect that model capacity has on attack success, we use fully-connected neural networks for which we can control both factors independently.  
We generate two synthetic data sets with $d =  50, 150$ respectively, and study to what extent an increase in model capacity, in this case the number of hidden nodes, impacts the performance of our attacks.
For $d=50$ the results in Figure \ref{fig:ann_synth_tpr_fpr_n10000_d50} suggest that increasing model capacity does not yield performance increases for our distance-based attacks.
On the other hand, when $d=150$ the results in Figure \ref{fig:ann_synth_tpr_fpr_n10000_d150} show that increasing model capacity drastically increases the attack success of the \texttt{CFD LRT}, while the simple \texttt{CFD} still does not work well.
For the model with the largest capacity, three layers of 1000 hidden nodes per layer, \texttt{CFD LRT} performs considerably better than even the \texttt{Loss LRT}.
Taken together these results suggest that a combination of high model capacity and high feature dimension increases the vulnerability of recourses to MI attacks, where feature dimension appears to play a  more important role.





%% file: tables/ann_performance_realworld.tex
\begin{table}[tb][ht]
\centering
\begin{tabular}{cccccccccc}
\toprule
\multirow{2}{*}{Data} & \multirow{2}{*}{Measures} & \multicolumn{3}{c}{\texttt{CFD}} & \multicolumn{3}{c}{\texttt{CFD LRT}} \\
\cmidrule(lr){3-8} 
& {} & \texttt{SCFE} & \texttt{GS} & \texttt{CCHVAE} & \texttt{SCFE} &   \texttt{GS} & \texttt{CCHVAE} \\
\midrule
\multirow{4}{*}{A} & AUC & 0.4971   & 0.5038  & 0.5008   & 0.4988 & \textbf{0.5103 } & 0.5066 \\
 & BA & 0.5115  & 0.5125  & 0.5056  & \textbf{0.5132} & 0.5098  &  0.5176  \\
 & TPR (0.1) & 0.1039  & 0.1020 & 0.1058 & 0.1010 & 0.1043 &  \textbf{0.1298} \\
 & TPR (0.01) & 0.0121 & 0.0097 & 0.0157 & \textbf{0.0158}  & 0.0095 & 0.0134  \\
\midrule
\multirow{4}{*}{H} & AUC & 0.5887  & 0.5410  & 0.4874  & 0.5829  & 0.5027 &\textbf{0.6789}  \\
 & BA & 0.5904 & 0.5404 & 0.5473  & 0.5924 & 0.5326  & \textbf{0.6389}  \\
 & TPR (0.1) & 0.1130 & 0.1223 & 0.0863 & 0.1106 & 0.1142 & \textbf{0.2635}  \\
 & TPR (0.01) & 0.0155 & 0.0176 & 0.0016  & 0.0135  & 0.0372 & \textbf{0.0513}  \\
\midrule
\multirow{4}{*}{D} & AUC & \textbf{0.5051}  & 0.5000  & NA   & 0.5050  & 0.5047  & NA \\
 & BA & 0.5100  & 0.5133  & NA & \textbf{0.5145} & 0.5136 & NA    \\
 & TPR (0.1) & 0.1020  & 0.0950 & NA & 0.0894 & \textbf{0.1181} & NA   \\
 & TPR (0.01) & 0.0093 & 0.0083 & NA & 0.0113  & \textbf{0.0159}  & NA \\
\bottomrule
\end{tabular}
\caption{Comparing the efficacy of the distance-based attacks for various recourse methods. The AUC denotes the area under the receiver operating characteristic curve, BA is the Balanced Accuracy, and $\text{TPR}(x)$ measures the TPR when the $\text{FPR}=x$. For the Diabetes data set \texttt{CCHVAE} could not identify any recourses.}
\vspace{-0.1in}
\label{table:comparing_efficacy}
\vspace{-0.1in}
\end{table}

%% file: figures/figures_ann_realworld.tex
\begin{figure}[!ht]
\centering
\begin{subfigure}{\columnwidth}
\centering
\scalebox{0.95}{%
\includegraphics[width=\textwidth]{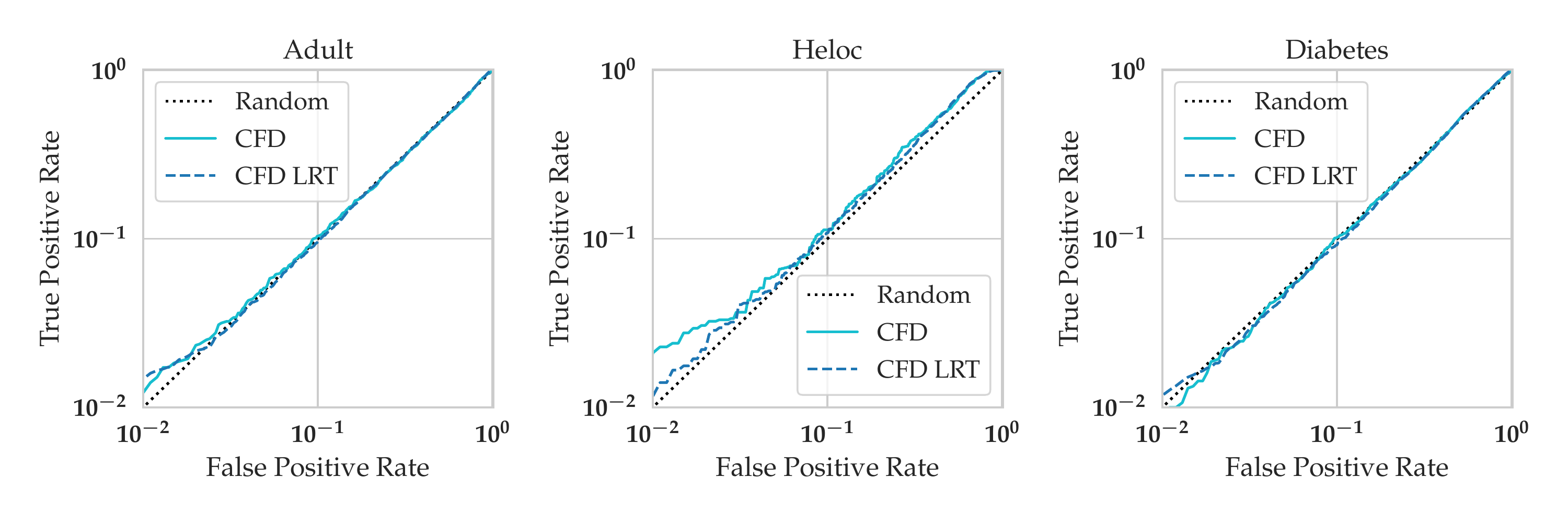}
}
\caption{\texttt{SCFE}}
\end{subfigure}
\vfill 
\begin{subfigure}{\columnwidth}
\centering
\scalebox{0.95}{%
\includegraphics[width=\textwidth]{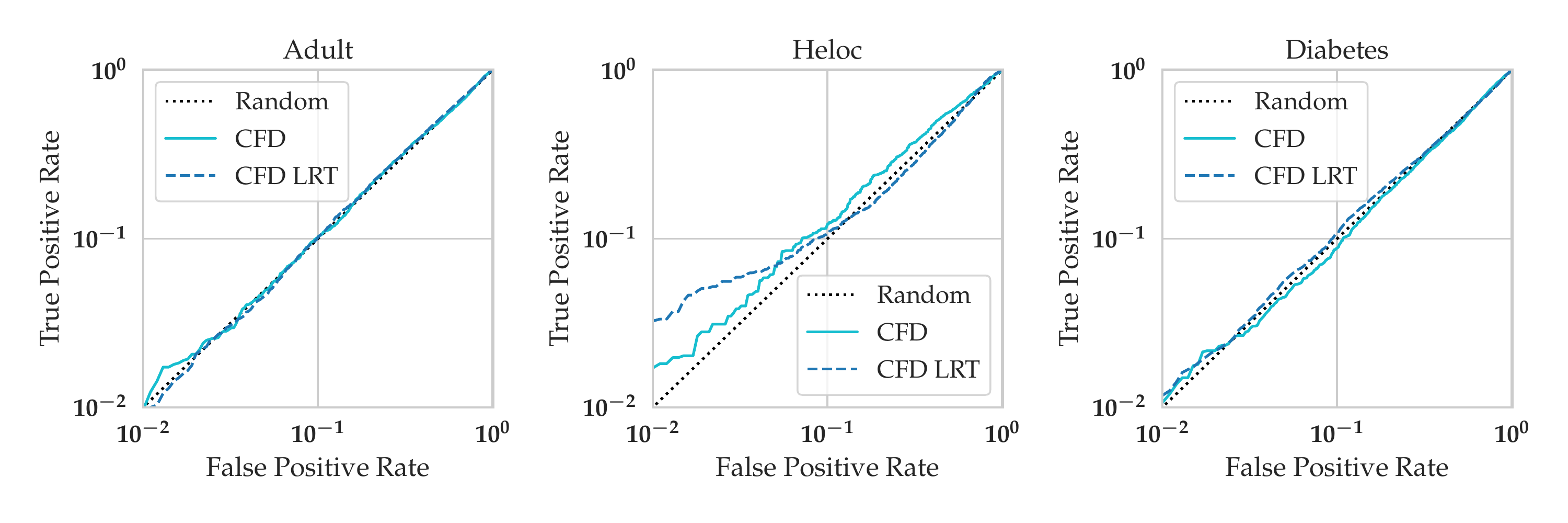}
}
\caption{\texttt{GS}}
\end{subfigure}
\vfill 
\begin{subfigure}{\columnwidth}
\centering
\scalebox{0.65}{%
\includegraphics[width=\textwidth]{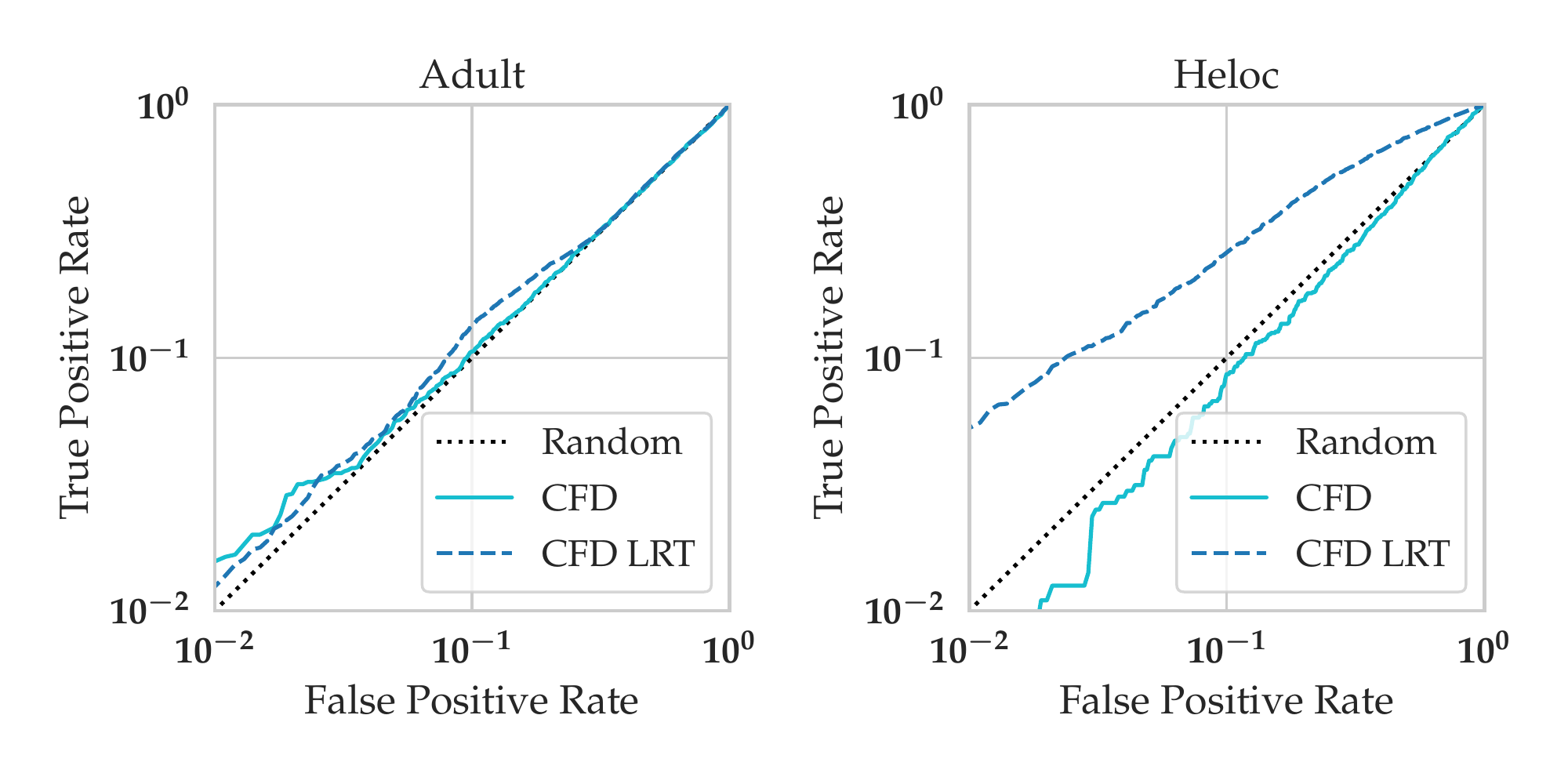}
}
\caption{\texttt{CCHVAE}}
\end{subfigure}
\caption{Comparing the attack efficacy across different MI attacks for fully connected neural network model trained on the real-world data sets when \texttt{SCFE} and \texttt{GS} are used for the attacks.
Small upwards deviations from the diagonal at low false-positive rates (i.e. 0.01) indicate that a small fraction of points can accurately be identified as members of the training data set.
}
\vspace{-0.2in}
\label{fig:real_world_all_methods}
\vspace{-0.1in}
\end{figure}

%% file: figures/figures_linear_synth.tex
\begin{figure*}[!ht]
\centering
\begin{subfigure}{\textwidth}
\centering
\scalebox{0.95}{%
\includegraphics[width=\textwidth]{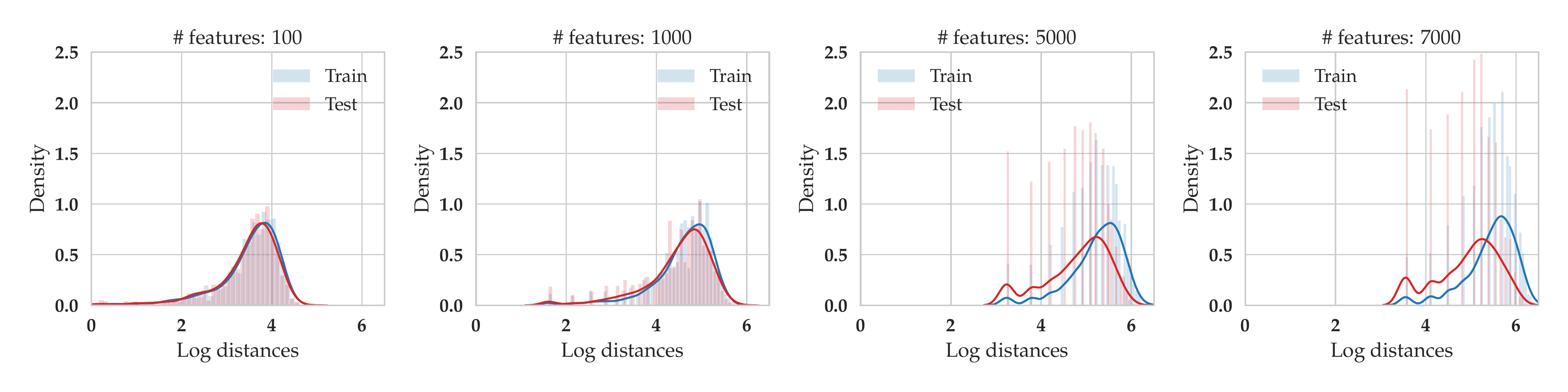}
}
\caption{Comparing log distances ($\ell_1$) to the decision boundary across train and test points.}
\label{fig:linear_distances}
\end{subfigure}
\vfill 
\begin{subfigure}{\textwidth}
\centering
\scalebox{0.95}{%
\includegraphics[width=\textwidth]{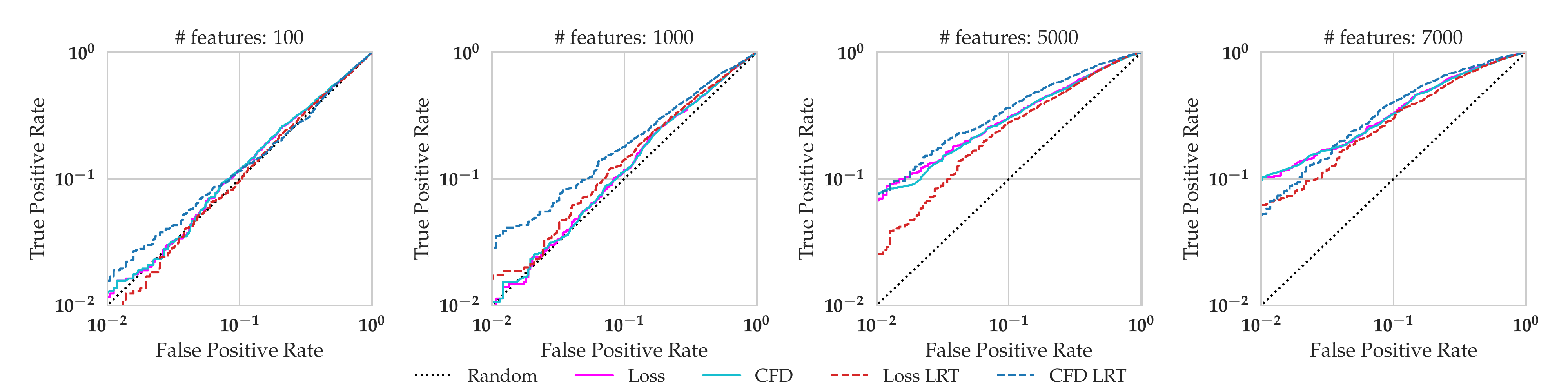}
}
\caption{Comparing the true positive rates against the false positive rates across different MI attacks using log-scaled ROC curves.}
\label{fig:linear_tpr_fpr}
\end{subfigure}
\caption{Demonstrating the efficacy of our proposed distance-based attack for logistic regression models trained on the synthetic data set when \texttt{SCFE} is used for the attack. At the interpolation threshold (i.e., when the number of training points equals the feature dimension: $d=n=5000$) the baseline loss-based and distance-based attacks start outperforming the lrt-based attacks.
}
\label{fig:linear_synth_experiments}
\end{figure*}

%% file: figures/figures_ann_synth.tex
\begin{figure*}[!ht]
\centering
\begin{subfigure}{\textwidth}
\centering
\scalebox{0.95}{%
\includegraphics[width=\textwidth]{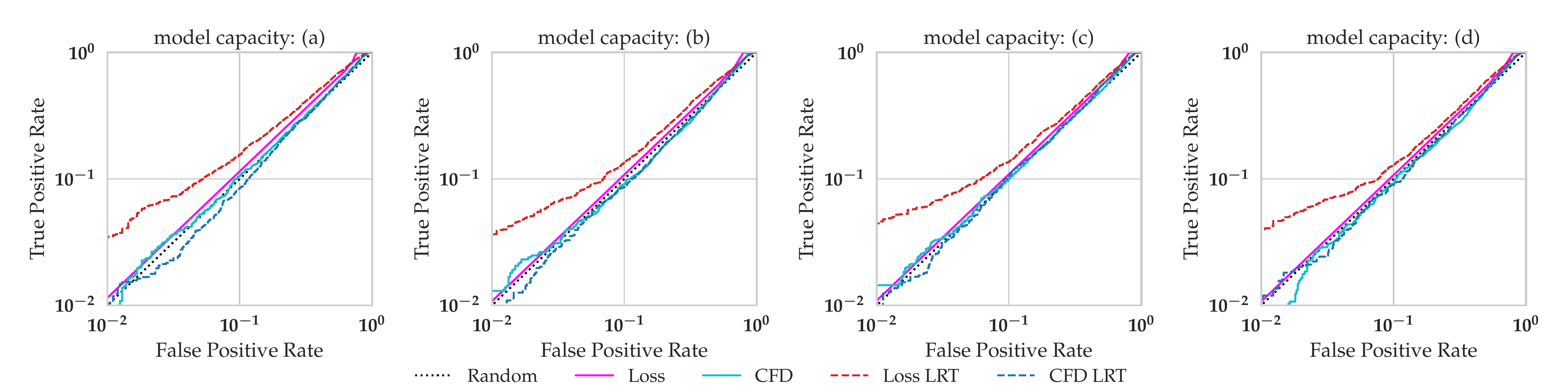}
}
\caption{$\text{\# features} = 50$}
\label{fig:ann_synth_tpr_fpr_n10000_d50}
\end{subfigure}
\vfill 
\begin{subfigure}{\textwidth}
\centering
\scalebox{0.95}{%
\includegraphics[width=\textwidth]{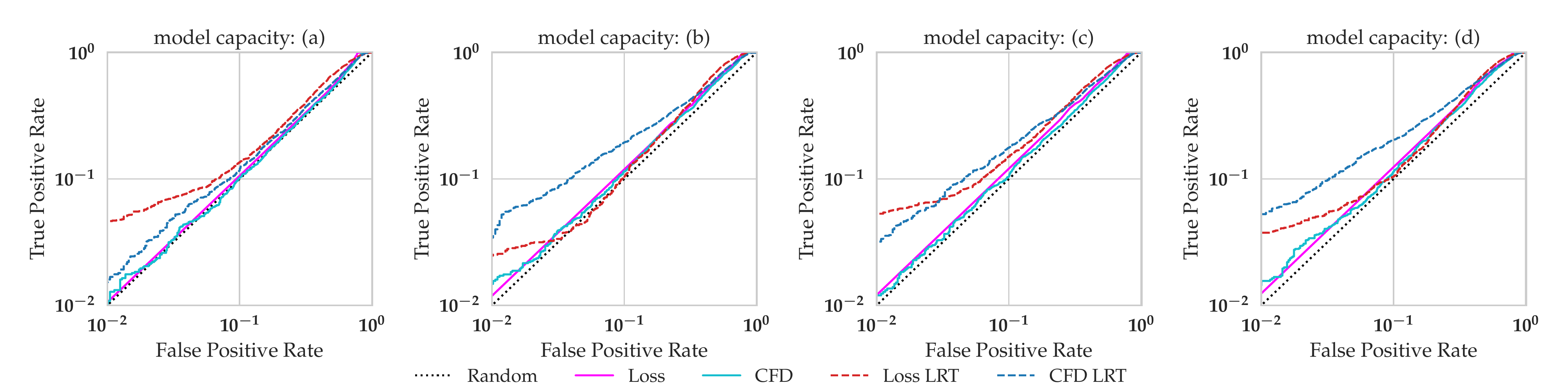}
}
\caption{$\text{\# features} = 150$}
\label{fig:ann_synth_tpr_fpr_n10000_d150}
\end{subfigure}
\caption{Demonstrating that both the network capacity and the number of features $d$ matter for the efficacy of the distance-based attack. We trained neural network models on $10000$ instances from the synthetic data set and used \texttt{SCFE} for the attack. 
From left to right the model capacity increases: (a): two-layer neural network with 1000 hidden nodes. (b): three-layer neural network with 100 hidden nodes in each hidden layer. (c): three-layer neural networks with 333 hidden nodes in each hidden layer. (d): three-layer neural networks with 1000 hidden nodes in each hidden layer.}
\label{fig:ann_synth_experiments_n10000}
\vspace{-0.2in}
\end{figure*}

%% file: conclusion.tex
\vspace{-0.1in}
In this work, we investigated the privacy risks associated with algorithmic recourse. More specifically, we introduced a general class of membership inference attacks called \emph{counterfactual distance-based attacks} which leverage algorithmic recourse to determine if an instance belongs to the training data of the underlying model or not. Empirical results on multiple synthetic and real-world datasets clearly demonstrate the efficacy of the proposed attacks, and highlight significant privacy leakage through recourses generated by a wide range of recourse methods. The proposed attacks also outperformed other state-of-the-art loss-based membership inference attacks on data with sufficiently high dimensionality. Overall, our results shed light on the critical risk of unintended privacy leakage through algorithmic recourse. 
Our work paves the way for other important future research directions. For instance, exploring solutions to mitigate the privacy risks highlighted in our work, either through heuristic methods, or by developing novel ways to generate differentially private recourses which are provably robust to such attacks. 

%% file: supplement.tex

\section{Proof of Theorem~\ref{priv_thm}}
\begin{proof}
Throughout the proof let the event that $\A$ recieves $(z, s)$ and outputs $G = \text{MEMBER}$ be denoted by $\mathcal{A}(z, s) =1$. First we prove the following simple lemma in the Appendix, which says that since the recourse is generated privately, the probability it takes on any value can't be changed by more than $e^{\epsilon}$ depending on whether a given point $z$ is in the training set. 
\begin{lemma}\label{lem:sup_priv}
If $\bo$ is any event, for any $z$: 
$$ \Pr_{s \sim \mathcal{R}, D_t \sim \mathcal{D}^n}[s \in \bo | z \in D_t] \leq e^{\epsilon}\Pr_{s \sim \mathcal{R}, D_t \sim \mathcal{D}^n}[s \in \bo|z \not \in D_t]
$$
\end{lemma}

\begin{proof}
Fix arbitrary $D_t\backslash\{z\} = (z_2, \ldots z_n)$. Expanding
\begin{equation}
    \Pr_{s \sim \mathcal{R}, D_t \sim \mathcal{D}^n}[s \in \bo | z \in D_t] = \int_{(z_2, \ldots z_n) \sim \mathcal{D}^{n-1}}\Pr_{s \sim \mathcal{R}}[s \in \bo | z, D_t\backslash\{z\} = (z_2, \ldots z_n)]\Pr_{\mathcal{D}}[(z_2, \ldots z_n)]
\end{equation}

By the definition of differential privacy, for arbitrary $z'$:
\begin{multline}
    \int_{(z_2, \ldots z_n) \sim \mathcal{D}^{n-1}}\Pr_{s \sim \mathcal{R}}[s \in \bo | z, D_t\backslash\{z\} = (z_2, \ldots z_n)]\Pr_{\mathcal{D}}[(z_2, \ldots z_n)] \leq \\
    \int_{(z_2, \ldots z_n) \sim \mathcal{D}^{n-1}}e^{\epsilon}\Pr_{s \sim \mathcal{R}}[s \in \bo | z', D_t\backslash\{z\} =  (z_2, \ldots z_n)]\Pr_{\mathcal{D}}[(z_2, \ldots z_n)]
\end{multline}

Since this holds for arbitrary $z'$, we have: 
\begin{multline}
\Pr[s \in \bo | z \in D_t] \leq \inf_{z'}\left(\int_{(z_2, \ldots z_n) \sim \mathcal{D}^{n-1}}e^{\epsilon}\Pr_{s \sim \mathcal{R}}[s \in \bo | z', D_t\backslash\{z\} = (z_2, \ldots z_n)]\Pr_{\mathcal{D}}[(z_2, \ldots z_n)]\right) \leq \\
\mathbb{E}_{z' \sim \mathcal{D}}\left[\int_{(z_2, \ldots z_n) \sim \mathcal{D}^{n-1}}e^{\epsilon}\Pr_{s \sim \mathcal{R}}[s \in \bo | z', D_t\backslash\{z\} = (z_2, \ldots z_n)]\Pr_{\mathcal{D}}[(z_2, \ldots z_n)]\right] = e^{\epsilon}Pr[s \in \bo|z \not \in D_t],
\end{multline}

as desired.

\end{proof}

Recall that $\text{BA} = \frac{\text{FPR} + \text{TNR}}{2} = \frac{\Pr[\mathcal{A}(z, s) = 1 | z \in D_t] + \Pr[\mathcal{A}(z, s) = 0 | z \not \in D_t]}{2}$, and let $a^{1} = \Pr[\mathcal{A}(z, s) = 1 | z \in D_t]$, and $a^{0} = \Pr[\mathcal{A}(z, s) = 0 | z \not \in D_t]$. Then:
\begin{multline}
    a^1 = \Pr[\mathcal{A}(z, s) = 1 | z \in D_t] = \int_{s \in \mathcal{S}}\int_{z \in \mathcal{X} \times \mathcal{Y}}\Pr[\mathcal{A}(z, s) = 1|z,s]\Pr[s|z \in D_{t}]\Pr_{\D}[z] \leq \\ 
\int_{s \in \mathcal{S}}\int_{z \in \mathcal{X} \times \mathcal{Y}}\Pr[\mathcal{A}(z, s) = 1|x,s]\Pr_{\D}[x](e^{\epsilon}\Pr[s|z \not \in D_t]) = \\
\int_{s \in \mathcal{S}}\int_{z \in \mathcal{X} \times \mathcal{Y}}\Pr[\mathcal{A}(z, s) = 1|x,s]\Pr_{\D}[x](e^{\epsilon}\Pr[s]) = \\
e^{\epsilon}\Pr[A(x,s) = 1 | x \not \in D_t] = e^{\epsilon}(1-a_0),
\end{multline}
where the inequality follows from Lemma~\ref{lem:sup_priv}. Rearranging $a^1 \leq e^{\epsilon}(1-a_0)$  we get that $a_0 \leq 1-e^{-\epsilon}a_1$. Hence: 
$$ 
\text{BA} \leq \frac{a^1 + (1-e^{-\epsilon}a^1)}{2} = \frac{1}{2} + \frac{a_1 ( 1-e^{-\epsilon})}{2}
$$
Since $a_1 \leq 1$ this gives the result, and we note that for small $\epsilon$ this can be improved to $\text{BA} \leq \frac{1}{2} + \frac{(2-e^{-\epsilon})(1-e^{-\epsilon})}{4}$.
\end{proof}

\section{TRAINING DETAILS}
\subsection{Classification Models}
\input{tables/linear_classification_accuracy}
\input{tables/ann_classification_accuracy_synthetic}
\input{tables/ann_classification_accuracy_realworld}

\subsection{Recourse methods}

\begin{itemize}
\item {\texttt{SCFE}}: As suggested in~\citet{wachter2017counterfactual}, an Adam optimizer is used to optimize the recourse objective. 
We obtain recourses using an $\ell_1$ distance function, and the binary cross entropy loss between the counterfactual label and the target. 

\item {\texttt{GS}}:
The explanation model uses a counterfactual search algorithm in the input space. 
Particularly, instances are sampled within an $\ell_1$-norm ball with search radius search radius $r_i$ until recourse is successfully obtained.
The search radius of the norm ball is increased until recourse is found. 

\item {\texttt{C-CHVAE}}: An autoencoder is additionally trained to model the data-manifold. 
The explanation model uses a counterfactual search algorithm in the latent space of the AE. 
Particularly, a latent sample within an $\ell_1$-norm ball with search radius $r_l$ is used until recourse is successfully obtained. 
The search radius of the norm ball is increased until recourse is found. 
All generative models use $8$ latent dimensions, and 20 nodes in the first and third hidden layer.
\end{itemize}

\section{ADDITIONAL EXPERIMENTAL RESULTS}
\begin{figure*}[h!]
\centering
\begin{subfigure}{\textwidth}
\centering
\scalebox{0.90}{%
\includegraphics[width=\textwidth]{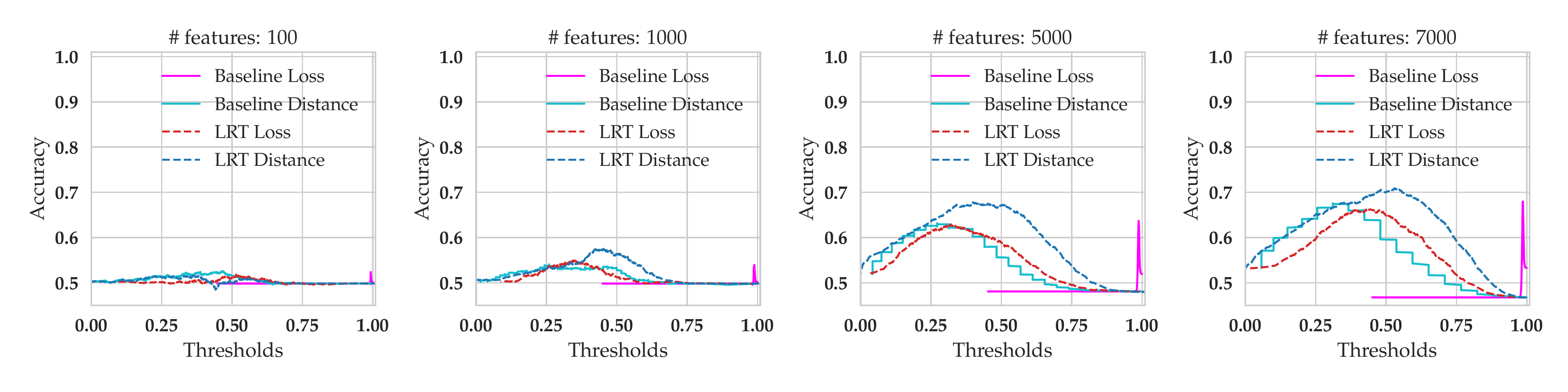}
}
\label{fig:linear_thresholds}
\end{subfigure}
\caption{Demonstrating the efficacy of our proposed distance-based attack for logistic regression models trained on the synthetic data set by showing the attack accuracy as the threshold varies. 
Both the LRT-Distance as well as the baseline distance attacks are the most competetive attacks.
The thresholds have been normalized to the range $[0,1]$.
}
\label{fig:linear_experiments}
\end{figure*}

\begin{figure*}[tb]
\centering
\begin{subfigure}{\textwidth}
\centering
\includegraphics[width=\textwidth]{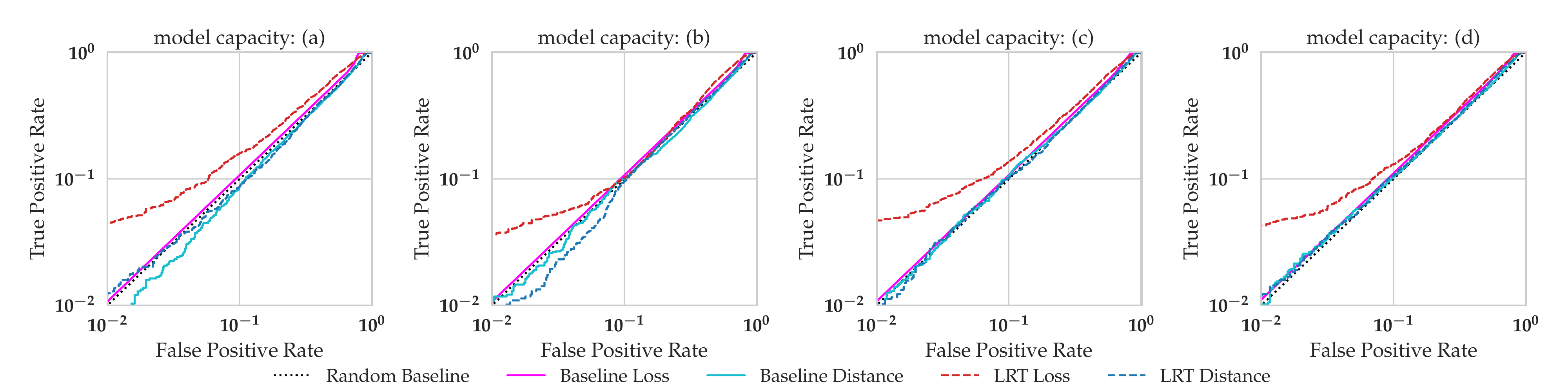}
\caption{$d=50, n=20000$}
\label{fig:ann_synth_tpr_fpr_n20000_d50}
\end{subfigure}
\vfill 
\begin{subfigure}{\textwidth}
\centering
\includegraphics[width=\textwidth]{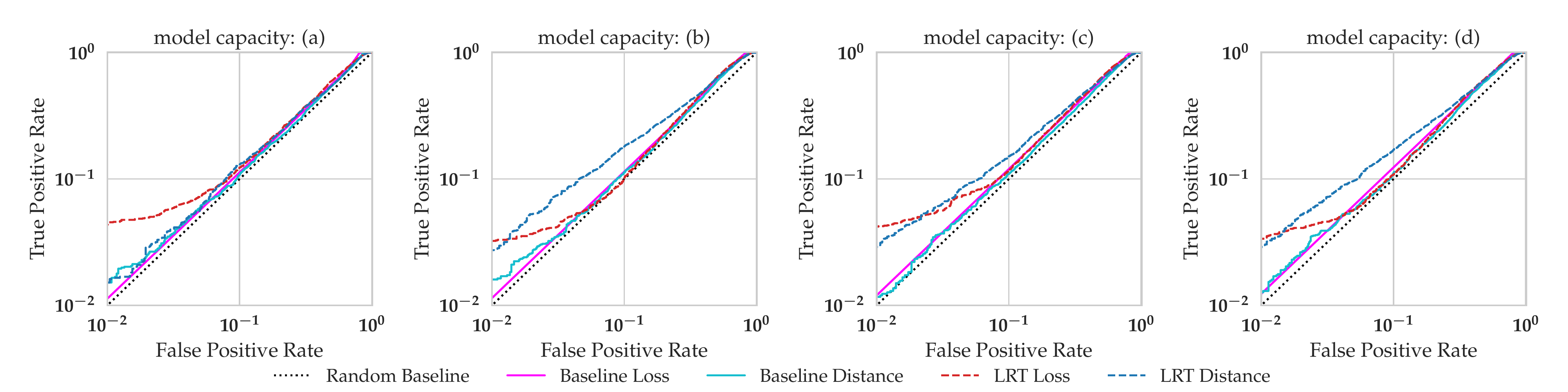}
\caption{$d=150, n=20000$}
\label{fig:fig:ann_synth_tpr_fpr_n20000_d150}
\end{subfigure}
\vfill
\begin{subfigure}{\textwidth}
\centering
\includegraphics[width=\textwidth]{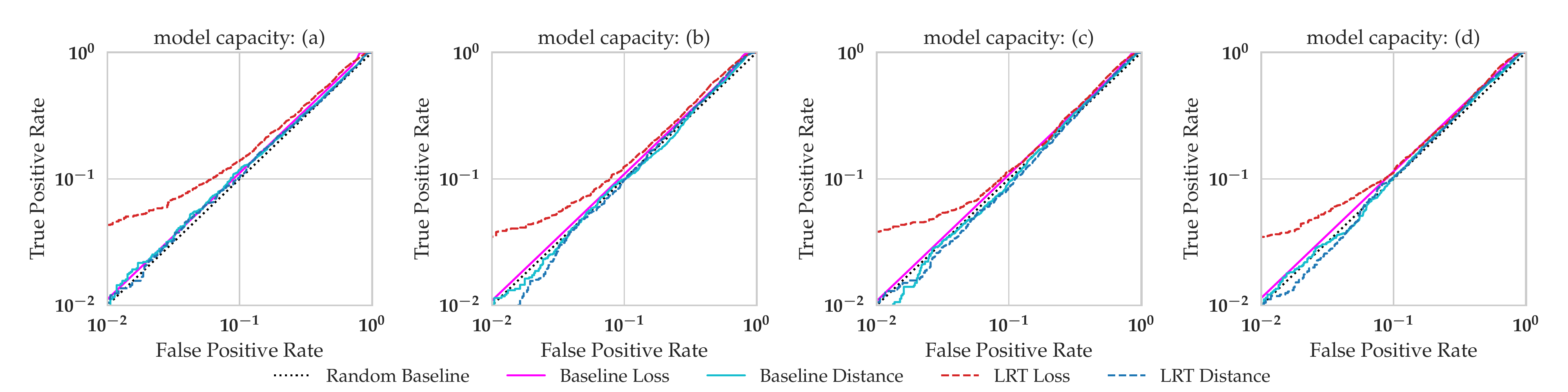}
\caption{$d=50, n=30000$}
\label{fig:fig:ann_synth_tpr_fpr_n20000_d50}
\end{subfigure}
\vfill 
\begin{subfigure}{\textwidth}
\centering
\includegraphics[width=\textwidth]{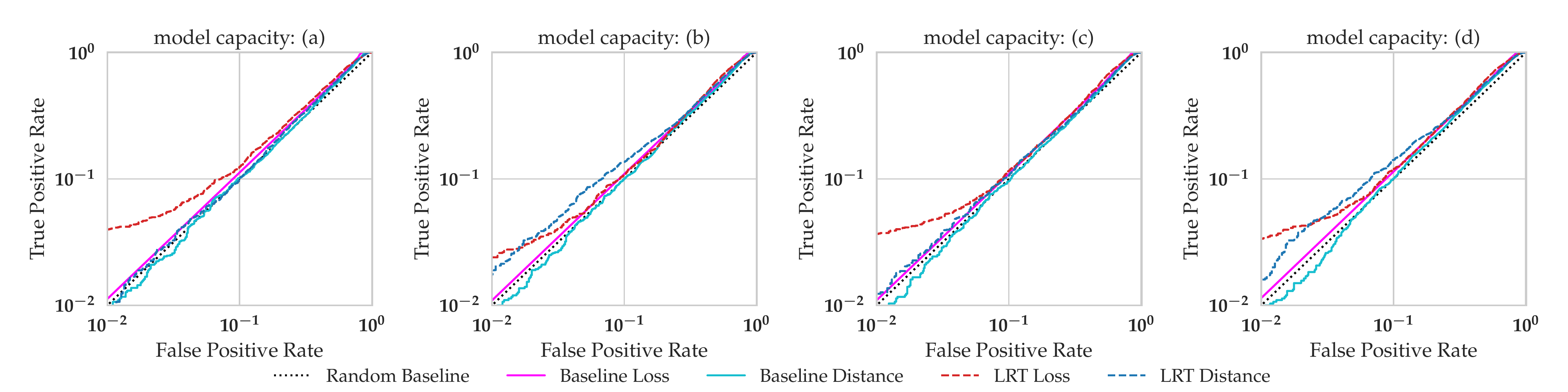}
\caption{$d=150, n=30000$}
\label{fig:fig:ann_synth_tpr_fpr_n30000_d150}
\end{subfigure}
\caption{Demonstrating the efficacy of our proposed distance-based attack for neural network models trained on the synthetic data set when \texttt{SCFE} is used for the attack. (a): two-layer network with 1000 hidden nodes. (b): three-layer neural network with 100 hidden nodes in each hidden layer. (c): three-layer neural networks with 333 hidden nodes in each hidden layer. (d): three-layer neural networks with 1000 hidden nodes in each hidden layer.
}
\label{fig:ann_synth_experiments_fpr_tpr_appendix}
\end{figure*}

\begin{figure*}[tb]
\centering
\begin{subfigure}{\textwidth}
\centering
\includegraphics[width=\textwidth]{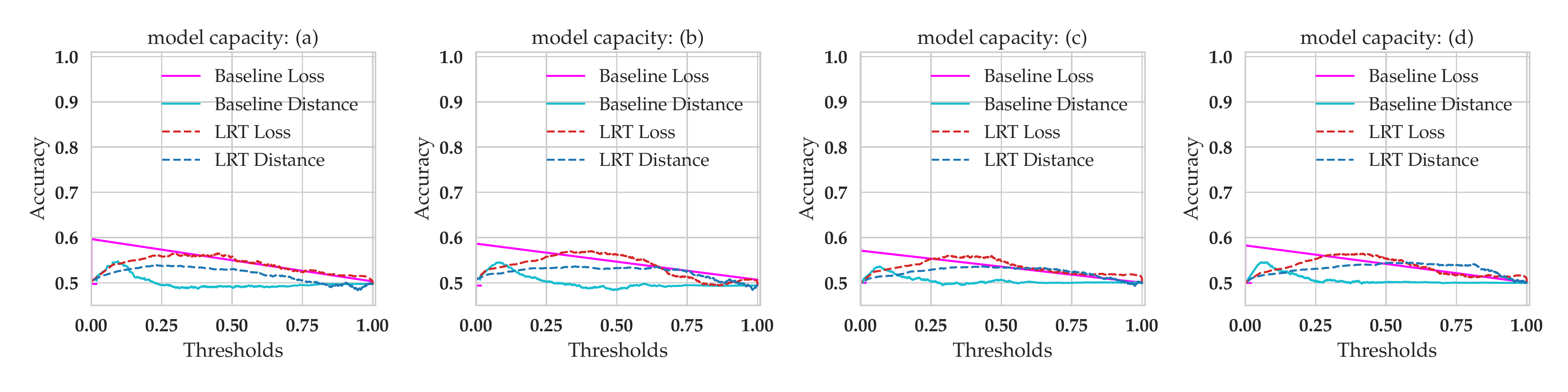}
\caption{$d=50, n=20000$}
\label{fig:ann_synth_thresholds_n20000_d50}
\end{subfigure}
\vfill 
\begin{subfigure}{\textwidth}
\centering
\includegraphics[width=\textwidth]{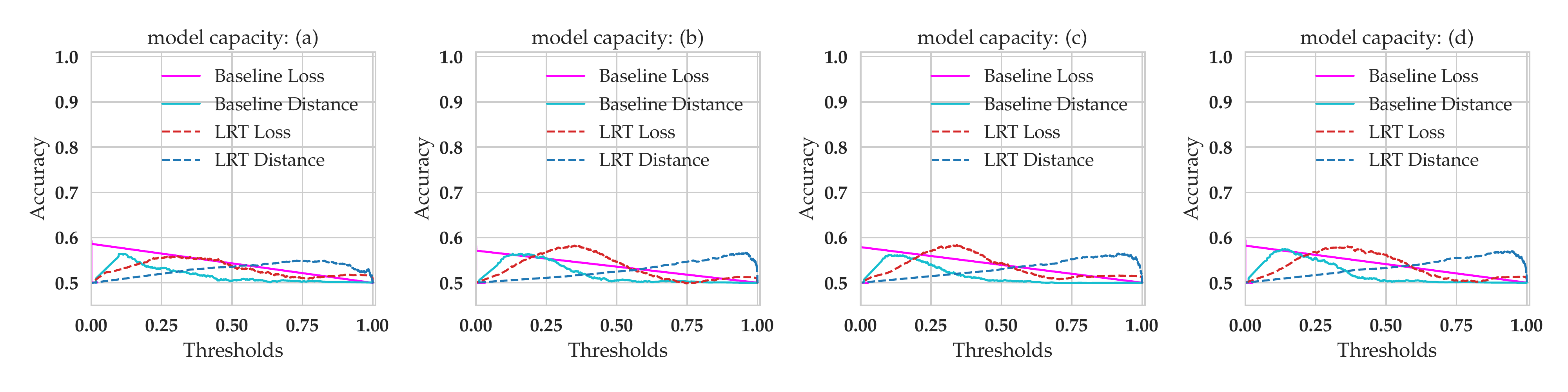}
\caption{$d=150, n=20000$}
\label{fig:fig:ann_synth_thresholds_n20000_d150}
\end{subfigure}
\vfill
\begin{subfigure}{\textwidth}
\centering
\includegraphics[width=\textwidth]{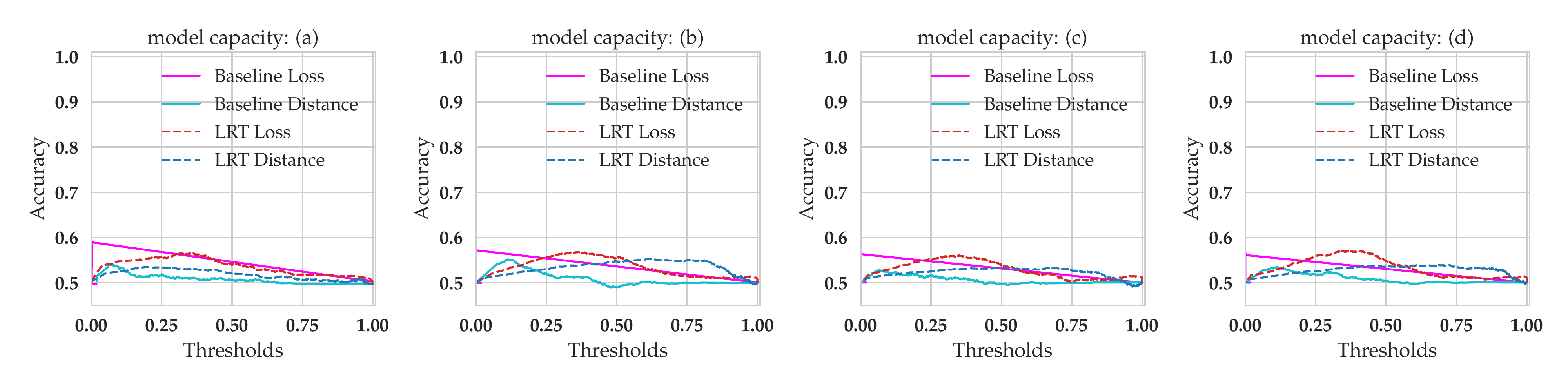}
\caption{$d=50, n=30000$}
\label{fig:fig:ann_synth_thresholds_n20000_d50}
\end{subfigure}
\vfill 
\begin{subfigure}{\textwidth}
\centering
\includegraphics[width=\textwidth]{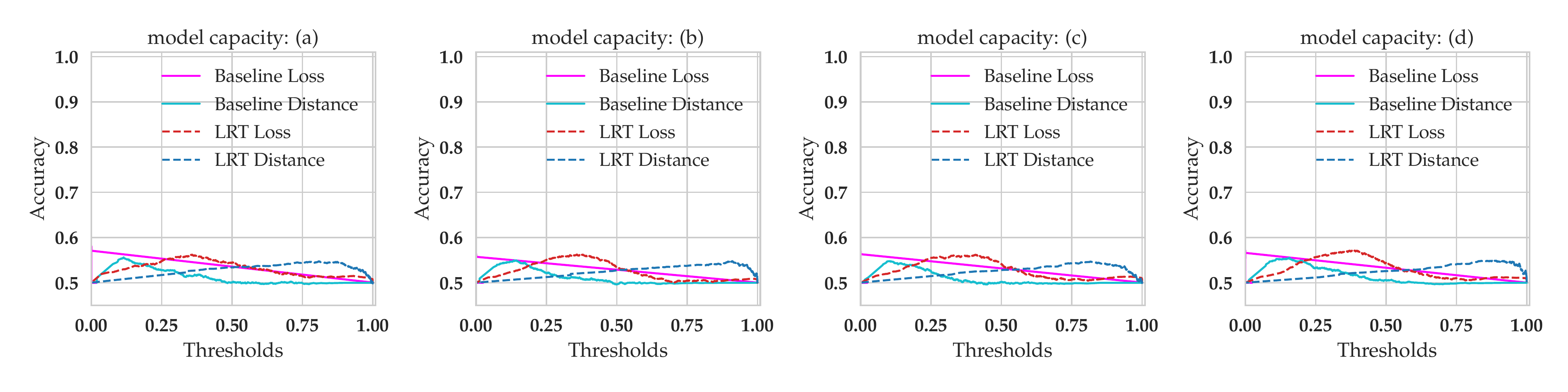}
\caption{$d=150, n=30000$}
\label{fig:fig:ann_synth_thresholds_n30000_d150}
\end{subfigure}
\caption{Demonstrating the efficacy of our proposed distance-based attack for neural network models trained on the synthetic data set when \texttt{SCFE} is used for the attack. (a): two-layer network with 1000 hidden nodes. (b): three-layer neural network with 100 hidden nodes in each hidden layer. (c): three-layer neural networks with 333 hidden nodes in each hidden layer. (d): three-layer neural networks with 1000 hidden nodes in each hidden layer. 
The thresholds have been normalized to the range $[0,1]$ when necessary.
}
\label{fig:ann_synth_experiments_fpr_tpr_appendix}
\end{figure*}

%% file: tables/linear_classification_accuracy.tex
\begin{table}[htb]
\centering
\begin{tabular}{ccccc}
\toprule
& \multicolumn{4}{c}{Model} \\
\cmidrule(lr){2-5}
& $d=100$ & $d=1000$ & $d=5000$ & $d=7000$  \\
\cmidrule(lr){1-5}
Train & 0.957 & 0.958 & 0.973 & 0.976 \\
Test & 0.951 & 0.936 & 0.899 & 0.872 \\
\bottomrule
\end{tabular}
\caption{Model performances for the logistic regression classifiers trained on the synthetic data sets.
We measured performance in terms of classification accuracy.}
\label{tab:classification_accuracy_synth_gauss_linear}
\end{table}

%% file: tables/ann_classification_accuracy_synthetic.tex
\begin{table}[htb]
\centering
\begin{tabular}{ccccccccccc}
\toprule
Data set & \multicolumn{4}{c}{$d=50$} & \multicolumn{4}{c}{$d=150$}  \\
 \cmidrule(lr){1-1} \cmidrule(lr){2-5} \cmidrule(lr){6-9}
Model size & (a) & (b) & (c) & (d) & (a) & (b) & (c) & (d)  \\
\cmidrule(lr){1-1} \cmidrule(lr){2-5} \cmidrule(lr){6-9}
Train & 1.00 & 1.00 & 1.00 & 1.00 & 1.00 & 1.00 & 1.00 & 1.00  \\
Test & 0.8632 & 0.8748 & 0.8820 & 0.8806 & 0.9060 & 0.9196 & 0.9130 & 0.9208 \\
\bottomrule
\end{tabular}
\caption{Model performances for the neural network classifiers trained on the synthetic data sets for different numbers of features.
We measured performance in terms of classification accuracy. (a): two-layer network with 1000 hidden nodes. (b): three-layer neural network with 100 hidden nodes in each hidden layer. (c): three-layer neural networks with 333 hidden nodes in each hidden layer. (d): three-layer neural networks with 1000 hidden nodes in each hidden layer.
}
\label{tab:classification_accuracy_synth_gauss_ann}
\end{table}

%% file: tables/ann_classification_accuracy_realworld.tex
\begin{table}[!ht]
\centering
\begin{tabular}{cccc}
\toprule
Data set & \multicolumn{1}{c}{Adult ($d=13$)} & \multicolumn{1}{c}{Heloc ($d=23$)} & \multicolumn{1}{c}{Diabetes ($d=42$)}  \\
 \cmidrule(lr){1-1} \cmidrule(lr){2-2} \cmidrule(lr){3-3} \cmidrule(lr){4-4}
Train & 0.9755 & 1.00 & 0.9096 \\
Test & 0.8109 & 0.6701 & 0.5334 \\
\bottomrule
\end{tabular}
\caption{Model performances for the neural network classifiers trained on the real-world data sets.
We measured performance in terms of classification accuracy for two-layer network with 1000 hidden nodes.}
\label{tab:classification_accuracy_synth_gauss_ann}
\end{table}